\documentclass{article}

\PassOptionsToPackage{numbers, compress}{natbib}
\usepackage{microtype}
\usepackage{graphicx}
\usepackage{subfigure}
\usepackage{booktabs} 
\usepackage{amsfonts} 
\usepackage{nicefrac} 
\usepackage{xcolor}   
\usepackage{soul} 
\usepackage{booktabs} 

\usepackage{hyperref}

\usepackage{arxiv}

\usepackage{amsmath}
\usepackage{amssymb}
\usepackage{mathtools}
\usepackage{amsthm}
\usepackage{lipsum}

\usepackage{tabularx}
\usepackage{multirow}

\usepackage[capitalize,noabbrev]{cleveref}

\theoremstyle{plain}
\newtheorem{theorem}{Theorem}[section]
\newtheorem{proposition}[theorem]{Proposition}

\theoremstyle{definition}
\newtheorem{definition}[theorem]{Definition}

\theoremstyle{remark}

\newcommand{\cf}[1]{\textcolor{red}{}}
\newcommand{\ds}[1]{\textcolor{blue}{}}

\usepackage{dsfont}

\newcommand{\AIM}{\texttt{AIM}} 
\newcommand{\AIMSSP}{\texttt{DD-SSP}} 
\newcommand{\SSP}{\texttt{SSP}} 
\newcommand{\AdaSSP}{\texttt{AdaSSP}} 
\newcommand{\AIMSynth}{\texttt{AIM-Synth}} 
 
\newcommand{\ObjPert}{\texttt{ObjPert}} 
\newcommand{\X}{\mathcal X}
\newcommand{\1}{\mathbf 1}
\newcommand{\R}{\mathbb R}
\newcommand{\boundx}{ \| \mathcal{X} \| }
\newcommand{\x}{\chi} 

\usepackage{algorithm}
\usepackage{algorithmic}

\title{Private Regression via Data-Dependent Sufficient Statistic Perturbation}

\author{%
  Cecilia Ferrando, \; Daniel Sheldon\\
  \\
  Manning College of Information and Computer Sciences\\
  University of Massachusetts, Amherst\\
  \texttt{\{cferrando\},\{sheldon\}@cs.umass.edu} \\
}

\begin{document}

\maketitle

\begin{abstract}
Sufficient statistic perturbation (SSP) is a widely used method for differentially private linear regression.
SSP adopts a data-independent approach where privacy noise from a simple distribution is added to sufficient statistics.
However, sufficient statistics can often be expressed as linear queries and better approximated by data-dependent mechanisms.
In this paper we introduce data-dependent SSP for linear regression based on post-processing privately released marginals, and find that it outperforms state-of-the-art data-independent SSP. 
We extend this result to logistic regression by developing an approximate objective that can be expressed in terms of sufficient statistics, resulting in a novel and highly competitive SSP approach for logistic regression.
We also make a connection to synthetic data for machine learning: for models with sufficient statistics, training on synthetic data corresponds to data-dependent SSP, with the overall utility determined by how well the mechanism answers these linear queries. 
\end{abstract}


\section{Introduction}
Differential privacy (DP)~\citep{dwork2006calibrating} 
is an established mathematical framework for protecting user privacy while analyzing sensitive data. A differentially private algorithm injects calibrated random noise into the data analytic process to mask the membership of single records in the data, limiting the information revealed about them in the output of the privatized algorithm.
The literature encompasses numerous methods for achieving differential privacy across a wide range of machine learning algorithms, including objective perturbation \citep{chaudhuri2008privacy, chaudhuri2011differentially, kifer2012private, jain2013differentially}, with applications to models trained via empirical risk minimization; gradient perturbation \citep{bassily2014private, abadi2016deep}, which is commonly used in deep learning and any models learning via gradient descent; one-posterior sampling (OPS)~\citep{dimitrakakis2017differential, wang2015privacy} with applications in private Bayesian inference; and finally, sufficient statistic perturbation (\SSP{}) \citep{vu2009differential, mcsherry2009differentially, dwork2010differential, zhang2016differential, foulds2016theory, wang2018revisiting, bernstein2019differentially, ferrando2022parametric}, with natural applications in exponential family estimation and linear regression.

\SSP{} adds calibrated random noise to the sufficient statistics of the problem of interest and uses the noisy sufficient statistics downstream to retrieve the target estimate. \SSP{} is appealing for a number of reasons. Sufficient statistics are by definition an information bottleneck, in that they summarize all the information about the model parameters~\citep{fisher1922mathematical}. For many models, like linear regression and exponential family distributions, their sensitivity is easy to quantify or bound, simplifying the DP analysis. Finally, they can be privatized via simple additive mechanisms, like the Laplace or Gaussian mechanism \citep{dwork2014algorithmic}. 

Existing \SSP{} methods are \emph{data-independent}, meaning they add noise to the sufficient statistics in a way that does not depend on the underlying data distribution. In a different branch of DP research, recent work has shown that \emph{data-dependent} mechanisms are the most effective for query answering and synthetic data \cite{hardt2012simple, zhang2017privbayes, gaboardi2014dual, aydore2021differentially, liu2021iterative, mckenna2022aim}.
In this paper, we develop \AIMSSP{}, a data-dependent \SSP{} method that uses private linear query answering to release DP sufficient statistics for linear regression. We use \AIM{} \cite{mckenna2022aim} as our query answering algorithm and show experimentally that \AIMSSP{} outperforms the state-of-the-art data-independent \SSP{} method \AdaSSP{} \citep{wang2018revisiting} (for discrete numerical data, as \AIM{} requires discrete data).

Furthermore, we extend the application of \AIMSSP{} to models without defined finite sufficient statistics by proposing a novel framework for logistic regression, where approximate sufficient statistics are derived and released in a data-dependent way to train the model by optimizing an approximate loss function. Our results show that for logistic regression tasks, \AIMSSP{} achieves better results than the widely used objective perturbation baseline.

Finally, we elaborate on the significance of our results with respect to the increasingly popular practice of training machine learning models on DP synthetic data. Our results support the observation that for these models training on synthetic data generated by linear-query preserving mechanisms effectively corresponds to a form of data-dependent \SSP{}.

\color{black}

\section{Background}
\subsection{Differential privacy} \label{dp-definitions}
Differential privacy (DP)~\citep{dwork2006calibrating} has become the preferred standard for preserving user privacy in data analysis, and it has been widely adopted by private and governmental organizations. Differential privacy allows many data computations (including statistical summaries and aggregates, and the training of various predictive models) to be performed while provably meeting privacy constraints. The concept of neighboring datasets is integral to differential privacy, which aims to limit the influence of any one individual on the algorithmic output in order to safeguard personal privacy. 
\begin{definition}[Neighboring datasets]
Two datasets $D$ and $D'$ are considered neighbors ($D \sim D'$) if $D'$ can be created by adding or deleting a single record from $D$.
\end{definition}
Based on the concept of neighboring datasets, we can define the \emph{sensitivity} of a function:
\begin{definition}[$L_2$ sensitivity]\label{def:sensitivity}
Given a vector-valued function of the data $f: \mathcal{D} \rightarrow \mathbb{R}^{p}$, the $L_{2}$ sensitivity of $f$ is defined as
$\Delta(f)=\max _{D \sim D^{\prime}}\left\|f(D)-f\left(D^{\prime}\right)\right\|_{2}$.
\end{definition}
Differential privacy can be achieved via different mechanisms of addition of calibrated random noise, with slightly different definitions. In this paper, we adopt $(\epsilon, \delta)$-DP, which can be achieved via the Gaussian Mechanism.

\begin{definition}[$(\epsilon, \delta)$-Differential Privacy]
    A randomized mechanism $\mathcal{M}$ : $\mathcal{D} \rightarrow \mathcal{R}$ satisfies $(\epsilon, \delta)$-differential privacy if for any neighbor datasets $D \sim D^{\prime} \in \mathcal{D}$, and any subset of possible outputs $S \subseteq \mathcal{R}$
\begin{align*}
\operatorname{Pr}[\mathcal{M}(D) \in S] \leq \exp (\epsilon) \operatorname{Pr}\left[\mathcal{M}\left(D^{\prime}\right) \in S\right]+\delta
\end{align*}
\end{definition}

\begin{definition}[Gaussian mechanism]
Let $f : \mathcal{D} \rightarrow \mathbb{R}^p $ be a vector-valued function of the input data. The Gaussian mechanism is given by
\begin{align*} \mathcal{M}(D) = f(D) + \nu\end{align*}
where $\nu$ is random noise drawn from $\mathcal{N}(0, \sigma^2 I_p)$ with variance $\sigma^2 = 2\ln(1.25/\delta)\cdot \Delta(f)^2/\epsilon^2$ and $\Delta(f)$ is the $L_2$-sensitivity of $f$. That is, the Gaussian mechanism adds i.i.d. Gaussian noise to each entry of $f(D)$ with scale $\sigma$ dependent on the privacy parameters.
\end{definition}

\begin{definition}[Post-processing property of DP] \citep{dwork2014algorithmic} \label{def:post-processing} If $\mathcal{M}(D)$ is $(\epsilon, \delta)$-DP, then for any deterministic or randomized function $g$, $g(\mathcal{M}(D))$ satisfies $(\epsilon, \delta)$-DP.
\end{definition}

\subsection{Differentially private synthetic data}
One of the most appealing applications of differential privacy is the creation of synthetic data, which is designed to be representative of the original data while maintaining privacy~\citep{hardt2012simple, zhang2017privbayes,  xie2018differentially, jordon2019pate, mckenna2019graphical, rosenblatt2020differentially, vietri2020new, liu2021iterative, aydore2021differentially, mckenna2021winning, vietri2022private, mckenna2022aim}. In this line of work, instead of perturbing the data or data analytic process, a model capturing the data distribution is estimated and used to sample new surrogate data, which can be safely used for the analytic tasks of interest. Research on differentially private synthetic data is ongoing, and two dedicated competitions have been hosted by the U.S. National Institute of Standards and Technology (NIST) \cite{NIST2018,NIST2020}, propelling advances in this area. Private synthetic data has many advantages, including fitting into existing data processing workflows and enabling users to answer multiple queries while maintaining privacy. 
Depending on the characteristics of the data captured in generative mechanism, the resulting synthetic data will be accurate specifically for tasks where those characteristics are relevant.

The key goal is therefore to tailor synthetic data to perform well on select classes of queries (workload) that are relevant to the problem of interest. The subject of generative synthetic data for machine learning tasks has recently gained traction in the literature~\citep{tao2021benchmarking, zhou2024bounding}, with the focus so far limited to benchmarking different synthetic data methods and bounding the empirical error it incurs when used for training machine learning models. However, an end-to-end analysis of the characteristics that make DP synthetic data suitable for machine learning tasks is needed. The following definitions allow us to formally elaborate on  differentially private synthetic data.

\begin{definition}[Dataset] A dataset $D$ is defined as a collection of $n$ potentially sensitive records. Each record $\x^{(i)} \in D$ is a $d$-dimensional vector $(\x^{(i)}_{1}, \ldots, \x^{(i)}_{d})$. 
\end{definition}

\begin{definition}[Domain] The domain of possible values for attribute $\x^{(i)}_j$ is $\Omega_j = \{1, \ldots, m_j\}$. The full domain of possible values for $\x^{(i)}$ is thus $\Omega=\Omega_{1} \times \cdots \times \Omega_{d}$ which has size $\prod_{j} m_{j}=m$. 
\end{definition}

We will later talk about numerical encodings of attributes (Section~\ref{sec:encoding}).

\begin{definition}[Marginals] Let $r \subseteq[d]$ be a subset of features, $\Omega_{r}=\prod_{j \in r} \Omega_{j}, m_{r}=\left|\Omega_{r}\right|$, and $\x_{r}=\left(\x_{j}\right)_{j \in r}$. 
The marginal on $r$ is a vector $\mu_r \in \mathbb{R}^{m_{r}}$ indexed by domain elements $t \in \Omega_{r}$ such that each entry is $\mu_r[t]=\sum_{\x \in D} \mathds{1}\left[\x_{r}=t\right]$ (i.e., counts). 
\end{definition}
With marginal queries, one record can only contribute a count of one to a single cell of the output vector. For this reason, the $L_{2}$ sensitivity of a marginal query $M_{r}$ is 1, regardless of the attributes in $r$. This facilitates the differential privacy analysis for marginal queries.

\begin{definition}[Workload]
A marginal workload $W$ is defined as a set of marginal queries $r_{1}, \ldots, r_{K}$ where $r_{k} \subseteq[d]$. 
\end{definition}

The goal of workload-based synthetic data generation models to minimize the approximation error on workload queries.

\subsection{PrivatePGM and AIM}
\emph{Marginal-based} approaches and \emph{select-measure-reconstruct} methods have generally emerged as effective methods for DP linear query answering and synthetic data generation~\citep{hay2009boosting, li2010optimizing, ding2011differentially, xiao2012dpcube, li2012adaptive, xu2013differentially, yaroslavtsev2013accurate, li2014data, qardaji2014priview, zhang2014towards, li2015matrix, mckenna2021hdmm}. The \emph{select-measure-reconstruct} paradigm operates by selecting a set of queries to measure using a noise addition mechanism to ensure their privacy. The noisy measurements can be used to estimate a target set of queries, called the \emph{workload}. Many differentially private synthetic data generation strategies use \texttt{Private-PGM}~\citep{mckenna2019graphical} to
post-process the noise-perturbed marginals and generate a synthetic dataset that respects them. 
\texttt{Private-PGM} can be used in the context of select-measure-reconstruct paradigms. 
Among these, \AIM{}~\citep{mckenna2022aim} has emerged as state-of-the-art by implementing a few key features. 
First, \AIM{} uses an intelligent initialization step to estimate one-way marginals. This results in a model where all one-way marginals are preserved well, and higher-order marginals can be estimated under an independence assumption. 
Second, \AIM{} uses a carefully chosen subset of the marginal queries and leverages the observation that lower-dimensional marginals exhibit a better signal-to-noise ratio than marginals with many attributes and low counts, and at the same time they can be used to estimate higher-dimensional marginal queries in the workload.
Third, the quality score function for selecting marginals to measure ensures that the selection is ``budget-adaptive'', i.e. it measures larger dimensional marginals only when the available privacy budget is large enough. 
In this work, we use \AIM{} as our query-answering algorithm of choice.
Note that \AIM{} uses zero-concentrated differential privacy (zCDP) \cite{bun2016concentrated}, an alternative privacy definition; the Gaussian mechanism as defined in Section~\ref{dp-definitions} satisfies $\frac{1}{2\sigma^2}$-zCDP \cite{bun2016concentrated}. In our experiments, we work with $(\epsilon, \delta)$-DP and the conversion to zCDP is handled internally in \AIM{}.

\section{Methods}

\begin{figure}[ht]
    \centering
    \includegraphics[width=1.\linewidth]{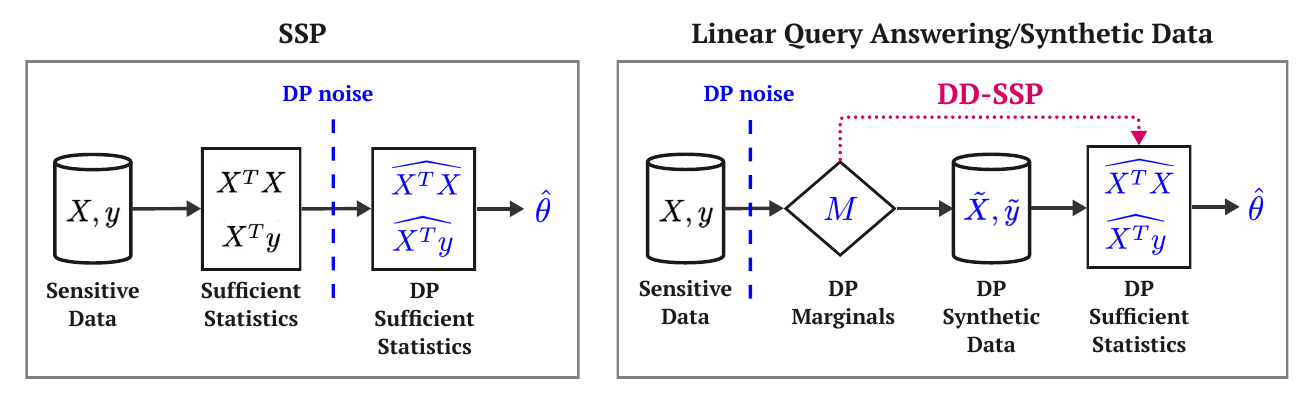}
    \caption{Diagram representing the \SSP{} data-independent workflow (left) vs the data-dependent linear query answering mechanism for marginal release/synthetic data workflow (right). Quantities indicated in blue follow privacy noise injection and are differentially private.
    }
    \label{fig:workflow}
\end{figure}

In this paper, we propose methods to privately estimate sufficient statistics in a data-dependent way. Specifically, our methods leverage privately released marginals computed by a data-dependent linear query answering algorithm (\AIM{}) to estimate sufficient statistics. We call this set of methods \AIMSSP{}. The main advantage of \AIMSSP{} is that it is \emph{data-dependent}. Our approach stems from the insight that, while data-independent noise addition is a simple and established approach to \SSP{}, many sufficient statistics can be expressed as linear queries, creating an opportunity to improve utility by using data-dependent query-answering DP mechanisms, which often achieve higher accuracy than simple additive noise mechanisms \cite{mckenna2022aim}. 
In fact, using private synthetic data to train certain models \emph{is already} a form of SSP. 
Figure~\ref{fig:workflow} compares a standard \SSP{} workflow and our proposed \AIMSSP{} workflows.
As seen in the figure, the pipeline of releasing synthetic data and then training a model via sufficient statistics can be viewed as a specific way of privatizing sufficient statistics for model training.

For linear regression, the application is straightforward: the problem has finite sufficient statistics and we demonstrate that two-way marginal queries are sufficient for their estimation. For other models, finite sufficient statistics are not available, but a polynomial approximation of the loss functions provides approximate sufficient statistics. This is the case for logistic regression, where finite sufficient statistics do not exist, but a Chebyshev polynomial approximation based on \cite{huggins2017pass} allows us to propose an \emph{approximate} version of the learning objective based on approximate sufficient statistics that can be expressed as linear queries, again retrievable via two-way marginal tables.

We use the synthetic data mechanism \AIM{} as our private query answering algorithm and modify its implementation it to output marginals directly, with no need to execute the synthetic data generation step. Depending on the input workload, \AIM{} will privately release marginals that preserve certain linear queries more accurately. We find that a \emph{two-way} marginal workload is sufficient for estimating or approximating the sufficient statistics for both linear and logistic regression. The proposed method is amenable to generalization beyond these classes of problems, and can be potentially extended to others by i) identifying or approximating their sufficient statistics, and ii) customizing the workload passed on as input in \AIM{} accordingly. Since our proposed methods are all based on post-processing DP workload query answers, the differential privacy analysis is straightforward (Definition~\ref{def:post-processing}). 

\color{black}

\subsection{Numerical encoding}
\label{sec:encoding} 
We assume discrete (or discretized) input data, which is a common format for tabular data and required by \AIM{} and other marginal-based approaches.
However, for machine learning, each record $\x$ must be mapped to a numerical vector $z = (x, y)$ where $x \in \R^p$ is a feature vector and $y \in \R$ is a target.
While the details of this encoding are often overlooked, they are important here for two reasons.
First, they are needed to tightly bound $\|x\|$, which is used in sensitivity calculations of a number of DP ML methods, with tighter bounds leading to higher utility.
Second, the encoding is a key part of recovering sufficient statistics from marginals. 

Let $\psi_j(\x_j) \in \R^{m_j}$ be the one-hot encoding of $\x_j$, i.e., the vector with entries $\psi_{j,s}(\x_j) = \1[\chi_j = s]$ for each $s \in \Omega_j$. We consider any numerical encoding of the form $\x_j \mapsto A_j \psi_j(\x_j)$ where $A_j \in \R^{p_j \times m_j}$ is a fixed linear transformation applied to the one-hot vector. 
This covers two special cases of interest.
The first is the scalar encoding with $A_j = v_j^T$ for a vector $v_j \in \R^{m_j}$ that specifies the numerical value for each $s \in \Omega_j$.
In this case the mapping simplifies to $\x_j \mapsto v_j(\x_j)$.
The second special case of interest is when $\Phi_j = I_j$ is the $m_j \times m_j$ identity matrix, so the mapping simplifies to $\x_j \mapsto \psi_j(\x_j)$ to give the one-hot encoding itself. 
Another common variation is a reduced one-hot encoding where $\Phi_j = \tilde{I}_j \in \R^{(m_j-1) \times m_j}$ is equal to $I_j$ with one row dropped to avoid redundant information in the one-hot encoding.

The full encoded record is $z = (z_{[j]})_{j=1}^d$ where $z_{[j]} = A_j \psi_j(\x_j) \in \R^{p_j}$ is the encoding of the $j$th attribute and these column vectors are concatenated vertically. A single entry of $z$ is selected as the target variable~$y$ leaving a feature vector $x$ of dimension $p := (\sum_{j=1}^d p_j) -1$.
Later, we will also use indexing expressions like $(\cdot)_{[j]}$ and $(\cdot)_{[j,k]}$ to refer to blocks of a vector or matrix corresponding to the encoding of the $j$th and $k$th attributes.

Let $z^{(i)} = (x^{(i)}, y^{(i)})$ denote the encoding of record $\chi^{(i)}$ and let $X \in \R^{n \times p}$ be the matrix with $i$th row equal to $(x^{(i)})^T$ and $y \in \R^n$ be the vector with $i$th entry equal to $y^{(i)}$. 
Many DP ML methods require bounds on the magnitude of the encoded data. 
Let $\boundx=\sup_{x \in \mathcal{X}}\|x\|$ and $\|\mathcal{Y}\|=\sup_{y \in \mathcal{Y}}|y|$ be bounds provided by the user where $\mathcal{X} \subset \R^{p}$ and $\mathcal{Y} \subset \R$ are guaranteed to contain all possible encoded feature vectors $x$ and target values $y$, respectively.
For example, a typical bound is $\|\X\| = \|x^+\|$ where $x^+_k \geq \sup |x_k|$ bounds the magnitude of a single feature. If $x_k$ is the scalar encoding of $\x_j$ we can take $x_k^+ = \max_{s \in \Omega_j} |v_j(s)|$. The following proposition describes how to tightly bound a feature vector that combines scalar features and one-hot encoded features. 
\begin{proposition}
\label{prop:boundx}
Suppose $x = (u, w)$ where $u \in \R^{a}$ satisfies $\|u\| \leq \|\mathcal U\|$ and $w \in \R^b$ contains the one-hot encodings (either reduced or not reduced) of $c$ attributes. Then $\|\X\| := \sqrt{\|\mathcal{U}\|^2 + c}$ is an upper bound on $\|x\|$.
\end{proposition}

\begin{proof}
$\|x\|^2 = \|u\|^2 + \|w\|^2 \leq \|\mathcal{U}\|^2 + c$ where $\|w\|^2 \leq c$ because $w$ is the concatenation of $c$ vectors each with at most a single entry of 1 and all other entries equal to 0.
\end{proof}

Suppose $u$ consists of scalar features and $\|\mathcal{U}\|$ is obtained by bounding each one separately as described above. This bound is tighter than the naive one of $\sqrt{\|\mathcal{U}\|^2 + b}$ that would be obtained by bounding each entry of the one-hot vectors separately.

\subsection{Linear regression}\label{sec:methods-linreg}
The goal of linear regression is to minimize the sum of squared differences between the observed values $y$ and predicted values $X \theta$ in a linear model with $\theta \in \R^p$. The ordinary least squares (OLS) estimator is obtained by minimizing the squared error loss function $\|y - X\theta\|^2$. Mathematically, the OLS estimator is given by $\hat{\theta} = (X^TX)^{-1}X^Ty$. In this context, the sufficient statistics are 
$T(X, y) = \{X^T X, X^T y\}$.
In \AIMSSP{}, we approximate $T(X,y)$ using linear queries. Specifically, we show that each entry of $X^T X$ and $X^T y$ can be obtained from pairwise marginals. 
The sufficient statistics we will consider all have the form of empirical second moments of the encoded attributes.

\subsubsection{Sufficient Statistics from Pairwise Marginals}

Let $Z = [X, y] \in \R^{n \times (p+1)}$. The matrix $Z^T Z$ has blocks that contain our sufficient statistics of interest:
\begin{equation}
Z^T Z = \begin{bmatrix}
X^T X & X^T y \\
y^T X & y^T y
\end{bmatrix}.
\label{eq:ZTZ-blocks}
\end{equation}
However, we will see that we can also construct $Z^TZ$ directly from marginals.

\begin{proposition}\label{prop:ZTZ}
Let $(Z^TZ)_{[j,k]}$ be the block of $Z^T Z$ with rows corresponding to the $j$th attribute encoding $z_{[j]} = A_j \psi_j(\x_j)$ and columns corresponding to the $k$th attribute encoding $z_{[k]} = A_j \psi_k(\x_k)$. Then 
\begin{align*}
(Z^T Z)_{[j,k]} = A_j \langle\mu_{j,k}\rangle A_k^T
\end{align*}
where $\langle \mu_{j,k} \rangle \in \R^{m_j \times m_k}$ is the $(j,k)$-marginal shaped as a matrix with $(s, t)$ entry $\mu_{j,k}[s, t] = \sum_{i=1}^n \1[\x^{(i)}_j = s, \x^{(i)}_k = t]$. Note that according to this definition $\langle \mu_{j,j}\rangle = \text{diag}(\mu_j)$.
\end{proposition}

This shows that we can reconstruct the sufficient statistic matrix $Z^T Z$ directly from the set of all single-attribute and pairwise marginals. Note that single-attribute marginals $\mu_j$ can be constructed from any $\mu_{j,k}$ with $k \neq j$.

\begin{proof}
The sufficient statistic matrix can be written as $Z^T Z = \sum_{i=1}^n z^{(i)} (z^{(i)})^T$. 
Indexing by blocks gives

{\allowdisplaybreaks
\begin{align*}
(Z^T Z)_{[j,k]} 
&= \sum_{i=1}^n z_{[j]}^{(i)} (z_{[k]}^{(i)})^T \\
&= \sum_{i=1}^n A_j \psi_j(\x_j^{(i)})\psi_k(\x_k^{(i)})^T A_k^T \\
&= A_j \Big(\sum_{i=1}^n \psi_j(\x_j^{(i)})\psi_k(\x_k^{(i)})^T\Big) A_k^T \\
&= A_j \langle\mu_{j,k}\rangle A_k^T,
\end{align*}}

In the last line, we used that $\psi_j(\x_j^{(i)})\psi_k(\x_k^{(i)})^T$ is a matrix with $(s, t)$ entry equal to $\1[\x_j^{(i)}=s, \x_k^{(i)}=t]$, so summing over all $i$ gives the matrix $\langle \mu_{j,k} \rangle$.
\end{proof}

Algorithm~\ref{alg:DPQuery} outlines how to retrieve approximate sufficient statistics $\widetilde{X^T X}$ and $\widetilde{X^T y}$ from marginals privately estimated by \AIM{}.

\begin{algorithm}
\begin{algorithmic}[1]
\caption{\AIMSSP{}}\label{alg:DPQuery}
\small
    \STATE \label{line:DPtablesOLS} $M \gets \AIM{}(\mathcal{D}, \epsilon, \delta)$ is the collection of privately computed pairwise marginal tables $\mu_{j, k}$ for all attribute pairs $(j, k)$.\\
    \STATE $(\widetilde{Z^T Z})_{[j,k]} \gets A_j \langle \mu_{j,k}\rangle A_k^T$ for all attribute pairs $(j, k)$ (see Proposition~\ref{prop:ZTZ})\\
    \STATE Extract $\widetilde{X^T X}$ and $\widetilde{X^T y}$ from $\widetilde{Z^T Z}$ using the block structure of Equation~\eqref{eq:ZTZ-blocks}
\end{algorithmic}
\end{algorithm}

This implies we can solve DP linear regression by i) retrieving $\widetilde{X^T X}$ and $\widetilde{X^T y}$ as outlined in Algorithm~\ref{alg:DPQuery}, and ii) finding $\hat\theta_{\text{DP}} = \widetilde{X^TX}^{-1} \widetilde{X^Ty}$.

\begin{proposition}\label{th:DP}
    \AIMSSP{} is ($\epsilon, \delta$)-DP.
\end{proposition}

The proof follows directly from the $(\epsilon, \delta)$-DP properties of the marginal-releasing algorithm (in our case, \AIM{}), and the fact that all subsequent steps are post-processing of a DP result (Definition~\ref{def:post-processing}).

\subsection{Logistic regression}

Logistic regression predicts the probability a binary label $y \in \{-1, +1\}$ takes value $+1$ as $p = 1/(1+\exp(-x \cdot \theta))$, where $\theta \in \R^p$ is a coefficient vector and $x \cdot \theta$ is the dot-product.
The log-likelihood function is
\begin{align*}
\ell(\theta) &= \sum_{i=1}^{n} \phi(x^{(i)}\cdot\theta y^{(i)}) 
\end{align*}
where $\phi(s) := -\log(1+e^{-s})$. 
Optimizing this log-likelihood is a convex optimization problem solvable numerically via standard optimizers. The log-likelihood does not have \emph{finite} sufficient statistics. However,~\cite{huggins2017pass} offers a polynomial strategy to obtain \emph{approximate} sufficient statistics for generalized linear models (GLMs), including logistic regression. 
\citet{kulkarni2021differentially} used a similar polynomial approximation for private Bayesian GLMs.
We propose a novel DP logistic regression method that combines two ideas: i) we use a Chebyshev approximation of the logistic regression log-likelihood based on~\cite{huggins2017pass}, which allows us to write the objective in terms of approximate sufficient statistics, and then ii) use \AIM{} privately released marginals to estimate the approximate sufficient statistics without accessing the sensitive data. This gives us the option to directly optimize an approximate log-likelihood based on privatized linear queries computed by \AIM{}. The choice of the input workload for \AIM{} depends on the characterization of the approximate log-likelihood. Based on our derivation below, we find that the best workload input for logistic regression is all pairwise marginals.

\citet{huggins2017pass} propose to approximate $\ell(\theta)$ by using an degree-$M$ polynomial approximation of the function $\phi$:
\begin{align*}
\phi(s) \approx \phi_{M}(s) := \sum_{m=0}^{M} {b_m}^{(M)} s^{m}
\end{align*}
where $b_j^{(M)}$ are constants. There are different choices for the orthogonal polynomial basis, and as in~\cite{huggins2017pass}, we focus on Chebyshev polynomials, which provide uniform quality guarantees over a finite interval $[-R, R]$ for a positive $R$. We can then write
\begin{align*}
\ell(\theta) &\approx \sum_{i=1}^{n} \sum_{m=0}^{M} b_m^{(M)} (x^{(i)} \cdot \theta y^{(i)})^{m}
\end{align*}

We choose to work with a degree-2 Chebyshev approximation over range $[-6, 6]$ (Figure \ref{fig:chebyshev}).
\begin{figure}
    \centering
    \includegraphics[width=0.4\linewidth]{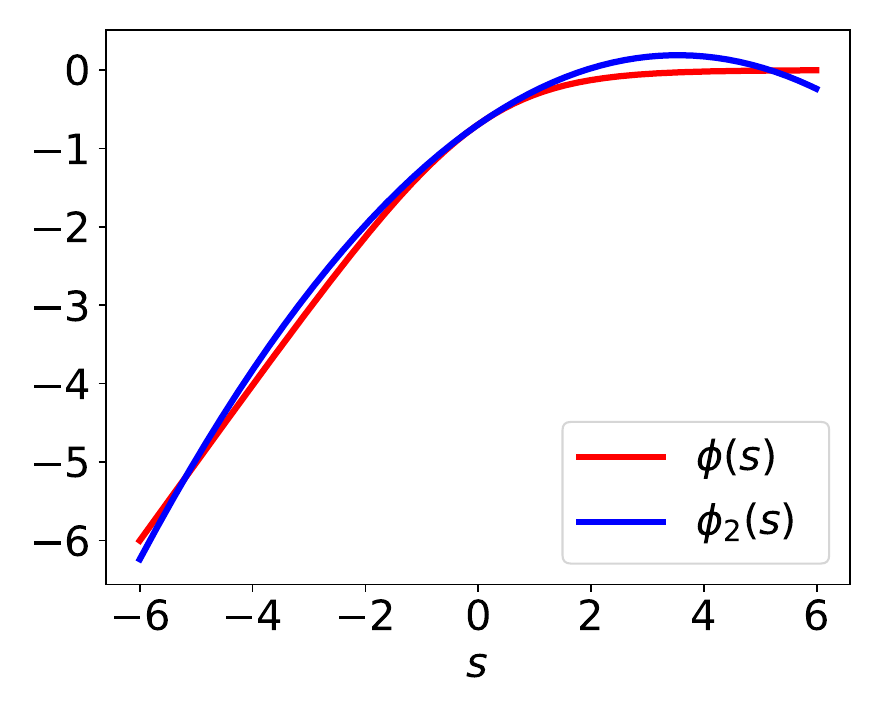}
    \caption{Degree 2 Chebyshev approximation of the logit function $\phi$, where $\phi (s) := - \log (1+e^{-s})$}
    \label{fig:chebyshev}
\end{figure}

\begin{proposition}\label{th:logregproposition}
    The logistic regression log-likelihood is approximated by second order Chebyshev polynomial 
    $\tilde\ell(\theta) \approx n b_0^{(2)} + b_1^{(2)} \theta \widetilde{X^{T} y}  +  b_2^{(2)} \cdot (\theta^T  \widetilde{X^T  X}  \theta )$, where $b_0^{(2)}, b_1^{(2)}, b_2^{(2)}$ are constants, and $\widetilde{X^T  X}$ and $\widetilde{X^{T} y}$ can be retrieved from pairwise marginals as in Proposition~\ref{prop:ZTZ}.
\end{proposition}

The proof is provided in Appendix~\ref{app:approx}. This allows us to define a logistic regression objective where $\widetilde{X^{T} X}$ and $\widetilde{X^{T} y}$ are obtained via Algorithm~\ref{alg:DPQuery}, which is DP by post-processing (Proposition~\ref{th:DP}). The objective can then be optimized directly via standard (non-private) procedures.

\subsection{Baselines} \label{sec:baselines}
We compare \AIMSSP{} to synthetic data method \AIMSynth{}, where private measurements are used to estimate the underlying distribution, and surrogate data is sampled from it. 
We train \AIMSynth{} with an input workload of all pairwise marginals to match the workload utilized for \AIMSSP{}. Since marginal-based synthetic data is designed to preserve the same linear queries that are sufficient to solve linear regression, or approximately sufficient for logistic regression, the expectation is that \AIMSSP{} will closely match the performance of \AIMSynth{}. As seen in Figure~\ref{fig:workflow}, the difference between these approaches is whether linear queries for sufficient statistics are computed directly from marginals estimated by the mechanism, or computed from synthetic data.

We also compare both methods against established DP baselines. Since our methods are based on privately reconstructing sufficient statistics, for DP linear regression sufficient statistic perturbation (\SSP{}) is the natural baseline choice. 
We choose \AdaSSP{} \citep{wang2018revisiting} for its competitive performance. As other \SSP{} methods, \AdaSSP{} uses limited data-adaptivity to add a ridge penalty based on an estimated bound on the eigenvalues of $X^T X$, but then adds independent noise to each entry of the sufficient statistics, unlike fully data-adaptive query-answering mechanisms.
The \AdaSSP{} algorithm is described in detail in Appendix~\ref{app:linreg}. 
For logistic regression, objective perturbation (\ObjPert{}) is a widely adopted solution originally proposed by \cite{chaudhuri2011differentially}, and further refined by \cite{kifer2012private} where it is extended to $(\epsilon, \delta)$-DP, with more general applicability and improved guarantees. 
The algorithm and details are provided in Appendix~\ref{app:logreg}. 

Both \AdaSSP{} and \ObjPert{} determine how much noise to add based on $\boundx$, which is the upper bound to the $L_2$-norm of any row $x$ of $X$ (Section~\ref{sec:encoding}).
For example, in \AdaSSP{}, $X^TX$ is noise-perturbed proportionally to $\boundx^2$.
From Proposition~\ref{prop:boundx}, we can set $\|\X\|^2 = \|\mathcal U\|^2 + c$ where $\|\mathcal U\|$ is a bound on the numerically-encoded features and $c$ is the number of one-hot-encoded attributes to obtain a tight sensitivity bound for these baselines.

\subsection{Limitations} \label{sec:limitations}
\AIMSSP{} takes privatized marginals as input.
When choosing \AIM{} as the mechanism for outputting DP marginals, we are limited to working with discrete data, which is a requirement in \AIM{} itself. 
Thus, our comparisons to other regression methods are scoped to discrete numerical data.
Future work may consider \AIMSSP{} with other mechanisms that support continuous data without discretization.
For the logistic regression approximation, we use Chebyshev second order polynomials; other approximation functions and/or degrees of precision could be evaluated, but are outside the scope of this paper. As shown in section \ref{app:experiments}, \AIMSSP{} demonstrates overall gains in regression accuracy compared to baseline methods, however this improvement comes at the cost of increased computational time, which is tied to the choice of \AIM{} as our private marginal releasing method. \AIM{}'s runtime increases with the size of the data domain and with $\epsilon$ (see Appendix \ref{app:experiments} for more detail). While our methods do introduce an increased computational burden, this tradeoff is justified by the substantial gains in accuracy. In terms of experimentation (see Section \ref{sec:experiments} below), for computational viability we study the variability of our results across 5 trials.

\section{Experiments} \label{sec:experiments}
In our experiments we evaluate the effectiveness of \AIMSSP{} for linear and logistic regression against \AdaSSP{} 
and \ObjPert{} respectively. Additionally, we assess the similarity between the performance of \AIMSSP{} and that of \AIM{} synthetic data (\AIMSynth{}), suggesting that data-dependent estimation (or approximation) of sufficient statistics explains the operational mechanisms that make marginal-based synthetic data suitable for machine learning tasks.\footnote{All experiment code is available at \url{https://github.com/ceciliaferrando/DD-SSP}.}

We run linear regression on real datasets Adult~\citep{misc_adult_2}, ACSIncome~\citep{ding2021retiring}, Fire~\citep{ridgeway2021challenge}, and Taxi~\citep{taxi}. For Adult, we choose `num-education' (number of education years) as the predicted variable. For ACSIncome, we queried the source database for the state of California and survey year 2018. The target variable is `PINCP' (income), discretized into 20 bins. For Fire, the target is `Priority' (of the call), and for Taxi it is `totalamount' (total fare amount of the ride).\footnote{The ACS data is sourced from \url{https://github.com/socialfoundations/folktables}. All other datasets are sourced from \url{https://github.com/ryan112358/hd-datasets}.} For every dataset, data is shuffled and split between $1,000$ data points as test set and the rest as the training set, capped at a maximum of $50,000$ train data points.
\AIM{} is trained with model size 200MB, maximum number of iterations 1,000, and a workload of all pairwise marginals. 
Other than \AIM{}, which uses the original feature levels, for training purposes non-numerical features are one-hot encoded as outlined in~\ref{sec:encoding}, dropping the first level to avoid multi-collinearity, and discrete numerical features keep the encoding from $1$ to $m_j$.
To calibrate the sensitivity accordingly, and avoid unnecessarily penalizing \AdaSSP{}, the bound $\|\X\|$ is computed as described in Sections \ref{sec:encoding} and \ref{sec:baselines}.
Numerical features are also rescaled to domain $[-1, 1]$ so that DP noise is allocated roughly evenly across statistics of different features.
We compare the Mean Squared Error (MSE) of DP query-based methods \AIMSSP{} and \AIMSynth{} 
against DP baseline \AdaSSP{} and the public baseline for $\epsilon \in \{0.05, 0.1, 0.5, 1.0, 2.0\}$, with a fixed $\delta = 1e-5$. The plots in Figure~\ref{fig:linlogplot} show that \AIMSSP{} and \AIMSynth{} have nearly identical performance and both improve significantly upon the DP baseline \AdaSSP{} on all datasets except ACSincome, where the performance is similar. 
\begin{figure*}[ht]
    \centering
    \makebox[\textwidth][c]{\includegraphics[width=1.\textwidth]{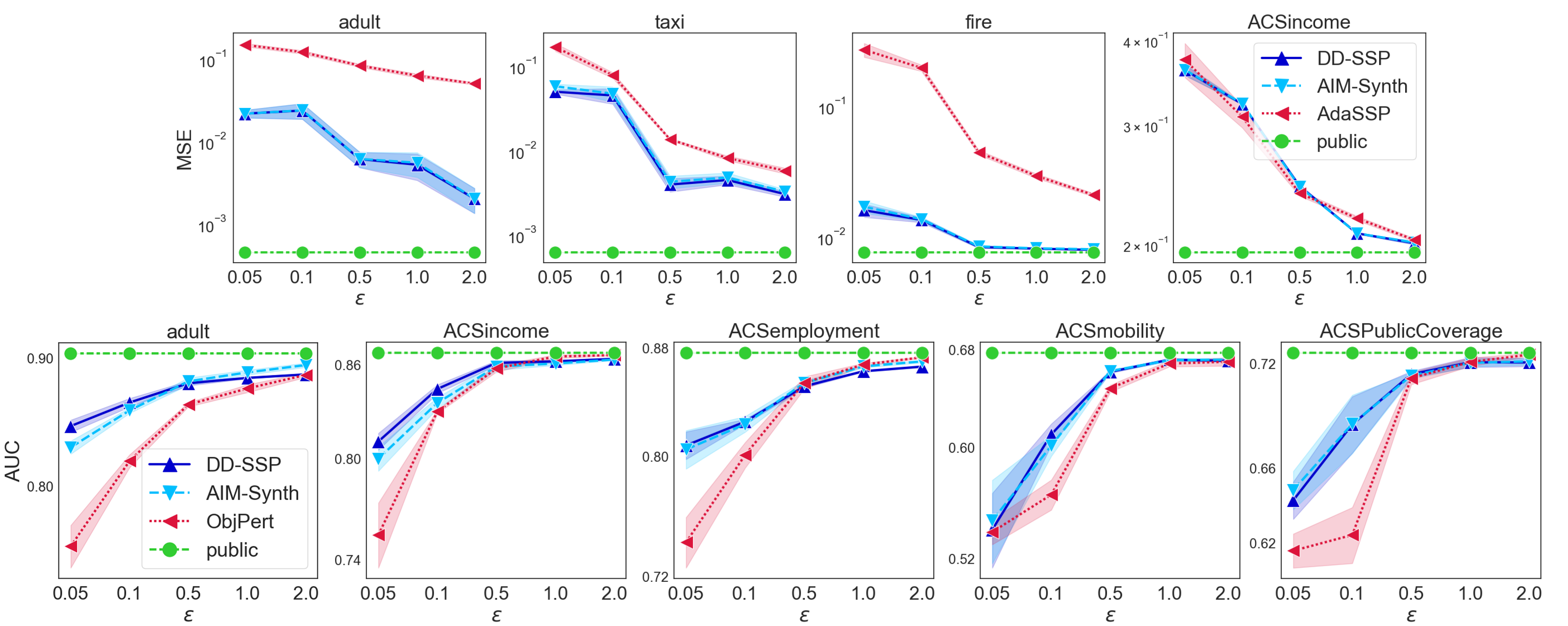}}
    \caption{Linear regression MSE results (first row) and logistic regression AUC results (second row). Standard error bars are computed over 5 trials.}
    \label{fig:linlogplot}
\end{figure*}
We run the same set of experiments for logistic regression on Adult, and four ACS datasets \citep{ding2021retiring}, ACSincome (target binary `PINCP' with cutoff at 50k), ACSmobility (`MIG'), ACSemployment (`ESR'), and ACSPublicCoverage (`PUBCOV'). All datasets are for the state of California and survey year 2018. 
The results in Figure \ref{fig:linlogplot} show that the performance of \AIMSSP{} closely matches \AIMSynth{}, surpassing it slightly in the case of Adult and ACSincome.
Both methods display overall improved AUC performance with respect to the \ObjPert{} baseline, with clear improvements in the low-$\epsilon$ regime, and close performance at higher $\epsilon$. 

Based on these results, we observe that:
i) \AIMSSP{} is a competitive option for DP linear and logistic regression, with better performance than data-independent \SSP{} for linear regression and \ObjPert{} for logistic regression;
ii) the same good performance can be achieved by training \AIMSynth{} with a workload of two-way marginals, which corresponds to estimating problem-specific data-dependent sufficient statistics. This suggests that the query-based sufficient statistics approximation explains the suitability of \AIM{} synthetic data for machine learning. At the same time, it implies \AIMSSP{} is a convenient and effective DP option to solve linear and logistic regression whenever pairwise marginals are available;
iii) the approximate \AIMSSP{} method for logistic regression constitutes a novel DP algorithm for solving logistic regression privately as an alternative to privatized ERM procedures.

\section{Discussion}
We introduce methods for data-dependent sufficient statistic perturbation (\AIMSSP{}). Our methods use privately released marginal tables to solve linear and logistic regression via sufficient statistics. We find that \AIMSSP{} performs better than data-independent SSP on linear regression and objective perturbation for logistic regression. Notably, the approximate \AIMSSP{} logistic regression algorithm is the first DP logistic regression method that allows analysts to solve logistic regression via a SSP algorithm, directly minimizing the approximate loss function. Additionally, we find that the performance of \AIMSSP{} is almost indistinguishable from that of \AIM{} synthetic data: this suggests that with the appropriate workload, training these machine learning models on query-based DP synthetic data corresponds to data-dependent SSP. 
Beyond linear regression, all exponential family distributions, including graphical models like Naive Bayes, have finite sufficient statistics, and for such models we can devise similar \AIMSSP{} solutions, or tailor synthetic data, by identifying a workload that supports the estimation of their sufficient statistics.
Future work can focus on developing approximate loss functions with finite sufficient statistics for a broader class of other models, including generalized linear models (GLMs). This will open the door to novel \AIMSSP{} methods that directly minimize the approximate loss functions, and improve utility by adding privacy noise in a data-dependent way. Additionally, methods targeting encoded workload $A_j \langle \mu_{j,k}\rangle A_k^T$ instead of $\mu_{j,k}$ can be explored and combined with advanced data-independent mechanisms like the matrix mechanism, potentially leading to a new class of encoding-aware \SSP{} methods. 

\section*{Acknowledgements}
This material is based upon work supported by the National Science Foundation under Grant No. 1749854, and by Oracle through a gift to the University of Massachusetts
Amherst in support of academic research.

\bibliography{references}

\newpage
\appendix
\onecolumn

\section{Logistic Regression log-likelihood approximation}\label{app:approx}
\begin{proof}
The log-likelihood for logistic regression can be expressed as
\begin{align*}
\ell(\theta) &= \sum_{i=1}^{n} \phi(x^{(i)} \cdot \theta y^{(i)}) 
\end{align*}

where $\phi(s) := -\log(1+e^{-s})$. Based on \cite{huggins2017pass}, we can approximate the logistic regression log-likelihood with a Chebyshev polynomial approximation of degree $M$:
\begin{align*}
\phi(s) \approx \phi_{M}(s) := \sum_{m=0}^{M} {b_m}^{(M)} s^{m}
\end{align*}
where $b_j^{(M)}$ are constants. Then,
\begin{align*}
\ell(\theta) &\approx \sum_{i=1}^{n} \sum_{m=0}^{M} b_m^{(M)} (x^{(i)} \cdot \theta y^{(i)})^{m}
\end{align*}
If we choose $M=2$, 
\begin{align*}
\ell(\theta) &\approx \sum_{i=1}^{n} b_0^{(2)} + b_1^{(2)} \cdot (x^{(i)} \cdot \theta y^{(i)}) + b_2^{(2)} \cdot (x^{(i)} \cdot \theta y^{(i)})^2
\end{align*}
The quadratic term is
\begin{align*}
     \sum_{i=1}^{n} (x^{(i)} \cdot \theta y^{(i)})^2 &= \sum_{i=1}^{n} {y^{(i)^2}} \sum_{j,k=1}^{d} x^{(i)}_j x^{(i)}_k \theta_j \theta_k = \sum_{j,k=1}^{d} \theta_j \theta_k \sum_{i=1}^n {y^{(i)^2}} x^{(i)}_j x^{(i)}_k
\end{align*}

Therefore we can rewrite the approximate log-likelihood as
\begin{align*}
\ell(\theta) &\approx n b_0^{(2)} + b_1^{(2)} \sum_{i=1}^{n}  x^{(i)} \cdot \theta y^{(i)} + b_2^{(2)} \sum_{j,k=1}^{d} \theta_j \theta_k \sum_{i=1}^n {y^{(i)}}^2 x^{(i)}_j x^{(i)}_k\\
&\approx n b_0^{(2)} + b_1^{(2)} y^{T} X \theta+  b_2^{(2)} \sum_{j,k=1}^{d} \theta_j \theta_k \sum_{i=1}^n x^{(i)}_j x^{(i)}_k\label{eq:approxll}\\
&\approx n b_0^{(2)} + b_1^{(2)} y^{T} X \theta+  b_2^{(2)}  \theta^T X^T X \theta 
\end{align*}
where the simplification in the last line follows from the fact that we work with $y^{(i)} \in \{1, -1\}$ and ${y^{(i)}}^2 = 1$ for any $i$.

$X^T X$ in the third term and $y^T X$ can be derived following the same proposed strategies as in linear regression (Section~\ref{sec:methods-linreg}), obtaining marginal query-based estimates $\widetilde{X^T  X}$ and $\widetilde{y^{T} X}$. We can then express the log-likelihood as
\begin{align*}
\ell(\theta) &\approx n b_0^{(2)} + b_1^{(2)} \widetilde{y^{T} X} \theta +  b_2^{(2)} \cdot (\theta^T  \widetilde{X^T  X}  \theta )
\end{align*}
\end{proof}

\section{Linear regression baseline}\label{app:linreg}

Algorithm \ref{alg:adassp} outlines the \AdaSSP{} method for linear regression \cite{wang2018revisiting}.
\begin{algorithm}
\begin{algorithmic}[1]
\label{alg:adassp}
\caption{AdaSSP \citep{wang2018revisiting}}
\STATE \textbf{Input:} Data $X$, $\boldsymbol{y}$. Privacy budget: $\epsilon$, $\delta$. Bounds: $\|\mathcal{X}\|$, $\|\mathcal{Y}\|$, $\rho \in (0,1)$ (0.05 in the paper)
\STATE \textbf{1.} Calculate the minimum eigenvalue $\lambda_{\min }\left(X^T X\right)$.
\STATE \textbf{2.} Privately release $\tilde{\lambda}_{\min }=\max \left\{\lambda_{\min }+\frac{\sqrt{\log (6 / \delta)}}{\epsilon / 3}\|\mathcal{X}\|^2 Z-\frac{\log (6 / \delta)}{\epsilon / 3}\|\mathcal{X}\|^2, 0\right\}$, where $Z \sim \mathcal{N}(0,1)$.
\STATE \textbf{3.} Set $\lambda=\max \left\{0, \frac{\sqrt{d \log (6 / \delta) \log \left(2 d^2 / \rho\right)}\|\mathcal{X}\|^2}{\epsilon / 3}-\tilde{\lambda}_{\min }\right\}$.
\STATE \textbf{4.} Privately release $\widehat{X^T X}=X^T X+\frac{\sqrt{\log (6 / \delta)}\|\mathcal{X}\|^2}{\epsilon / 3} Z$ for $Z \in \mathbb{R}^{d \times d}$ is a symmetric matrix and every element from the upper triangular matrix is sampled from $\mathcal{N}(0,1)$.
\STATE \textbf{5.} Privately release $\widehat{X^T y}=X^T y+\frac{\sqrt{\log (6 / \delta)}\|\mathcal{X}\|\|\mathcal{Y}\|}{\epsilon / 3} Z$ for $Z \sim \mathcal{N}\left(0, I_d\right)$.
\STATE \textbf{Output:} $\tilde{\theta}=\left( \widehat{X^T X}+\lambda I \right)^{-1} \widehat{X^T y}$.
\end{algorithmic}
\end{algorithm}

To reason about the sensitivity of $X^TX$, consider two neighboring datasets $X \in \mathbb{R}^{n \times d}$ and $X' \in \mathbb{R}^{(n+1) \times d}$ differing by one data entry $v \in \mathcal{X}$, where $v$ is a $d \times 1$ vector. Then,
\begin{align*}\Delta_{X^T X} = \text{sup}_{X \sim X'} \Vert f(X') - f(X) \Vert_F\end{align*}
Since $X$ and $X'$ only differ by one row ($v$), then $f(X') - f(X) = vv^T$~\citep{sheffet2017differentially}. 

So the sensitivity is maximum over $v$ of $\|\text{vec}(vv^T)\| = \|vv^T\|_F$.
We have

\begin{align*}
\Delta_{X^T X}^2 &=  \sup_{v \in \X} \|vv^T\|_F^2 \\
&= \sup_{v \in \X} \sum_{i=1}^d \sum_{j=1}^d (v_i v_j)^2 \\
&= \sup_{v \in \X} \left(\sum_{i=1}^d v_i^2 \right) \left(\sum_{j=1}^d v_j^2\right) \\
&= \sup_{v \in \X} \|v\|^4 \\
&= \|\X\|^4
\end{align*}

where $\Vert \mathcal{X} \Vert$ is the greatest possible norm of a vector in the domain $\mathcal{X}$. 
Therefore,
\begin{align*}
\Delta_{X^T X} = \|\X\|^2.
\end{align*}

The sensitivity of $X^Ty$ can be similarly derived.
Given neighboring datasets $X \in \mathbb{R}^{n \times d}, y \in \mathbb{R}^n$, and $X' \in \mathbb{R}^{(n+1) \times d}, y' \in \mathbb{R}^{n+1}$, where $v \in \mathcal{X} \subset \mathbb{R}^d$ is the new row, and $w \in \mathcal{Y} \subset \mathbb{R}$ is the new value in $y'$. Then,
\begin{align*}
 \Vert f\left(X', y'\right)-f(X, y) \Vert = \Vert X'^T y'-X^T y \Vert = \Vert w v \Vert
\end{align*}
Since $\|\mathcal{X}\|=\sup_{x \in \mathcal{X}}\|x\|$ and $\|\mathcal{Y}\|=\sup _{y \in \mathcal{Y}}|y|$, we have
\begin{align*}
\Delta_{X^T y} &=\sup _{(X, y) \sim\left(X', y'\right)}\left\|f\left(X', y'\right)-f(X, y)\right\|\\
&=\sup _{w \in \mathcal{Y}, v \in \mathcal{X}} |w| \cdot\|v\| = \Vert \mathcal{Y} \Vert \cdot \Vert \mathcal{X} \Vert
\end{align*}

\section{Logistic regression baseline}\label{app:logreg}
Our DP logistic regression baseline is based on the generalized objective perturbation algorithm in \cite{kifer2012private} (Algorithm~\ref{alg:genobjpert}). In this section, to match the notation in \cite{kifer2012private} $\ell\left(\theta ; z\right) = \log(1 + \exp(-x \cdot \theta y))$ is the loss for a single datum and $\hat {\mathcal L}(\theta; \mathcal D) = \frac{1}{n} \sum_{i=1}^n \ell(\theta; z^{(i)})$ is the average loss over the dataset.

\begin{algorithm} \label{alg:genobjpert}
\caption{Generalized Objective Perturbation Mechanism (\ObjPert{}) \citep{kifer2012private}}
\label{alg:obj-pert}
\begin{algorithmic}[1]
\REQUIRE dataset $\mathcal{D}=\left\{z^{(1)}, \ldots, z^{(n)}\right\}$, where $z^{(i)} = (x^{(i)}, y^{(i)})$, privacy parameters $\epsilon$ and $\delta$ ($\delta=0$ for $\epsilon$-differential privacy), bound $\boundx$ on the $L_2$ norm of any $x$ entry, convex regularizer $r$, a convex domain $\mathbb{F} \subseteq \mathbb{R}^d$, convex loss function $\hat{\mathcal{L}}(\theta ; \mathcal{D})=\frac{1}{n} \sum_{i=1}^n \ell\left(\theta ; z^{(i)}\right)$, with continuous Hessian, $\|\nabla \ell(\theta ; z)\| \leq \boundx$ (for all $z \in \mathcal{D}$ and $\theta \in \mathbb{F}$), and the eigenvalues of $\nabla^2 \ell(\theta ; z)$ bounded by $\frac{\boundx ^ 2}{4}$ (for all $z$ and for all $\theta \in \mathbb{F}$).
\STATE Set $\Delta \geq  \frac{\boundx ^ 2}{2\epsilon} $.
\STATE Sample $b \in \mathbb{R}^d$ from $\nu_2(b ; \epsilon, \delta, \boundx)=\mathcal{N}\left(0, \frac{\boundx^2\left(8 \log \frac{2}{\delta}+4 \epsilon\right)}{\epsilon^2} I_{d}\right)$.
\STATE $\hat\theta_{\text {DP}} \equiv \arg \min _{\theta \in \mathbb{F}} \hat{\mathcal{L}}(\theta ; \mathcal{D})+\frac{1}{n} r(\theta)+\frac{\Delta}{2 n}\|\theta\|^2+\frac{b^T \theta}{n}$.
\end{algorithmic}
\end{algorithm}

The algorithm requires the following bounds for the gradient and Hessian of $\ell$:
\begin{align*}
\|\nabla \ell(\theta; z)\| &\leq \| \mathcal{X} \| \\
\lambda_{\max}(\nabla^2\ell(\theta;z)) &\leq \frac{\| \mathcal{X} \|^2}{4}
\end{align*}

To reason about the sensitivity bounds, let $\phi(s) = \log(1+e^s)$.
 Then we can write $\ell(\theta; z) = \phi(- x \cdot \theta y)$. Following~\cite{Gower}, it is straightforward to derive that

\begin{align*}
\phi'(s) &= \frac{e^s}{1+e^s} \leq 1 \\
\phi''(s) &= \frac{e^s}{(1+e^s)^2} \leq \frac{1}{4}
\end{align*}

and clear that both quantities are non-negative. Then the gradient of $\ell$ is:
\begin{align*}
\nabla \ell(\theta; z) = \nabla_\theta \phi(- x \cdot \theta y) = \phi'(- x \cdot \theta y)\cdot -y x
\end{align*}
The norm is bounded as
\begin{align*}\| \nabla \ell(\theta; z)\| = |\phi'(- x \cdot \theta y)|\cdot |y|  \cdot \|x\| \leq \boundx
\end{align*}
where the inequality holds since $|\phi'(s)| \leq 1$ for all $s$ and $|y| = 1$.

By differentiating the gradient again and using the fact that $y^2 = 1$, we can derive the Hessian as:
\begin{align*}
\nabla^2 \ell(\theta; z) = \phi''(- x \cdot \theta y) x x^T
\end{align*}
The maximum eigenvalue is
\begin{align*}
\lambda_{\max}(\nabla^2 \ell(\theta; z)) 
&= \phi''(- x \cdot \theta y) \lambda_{\max}(x x^T) \\
&= \phi''(- x \cdot \theta y) \|x\|^2 \\
&\leq \frac{\boundx^2}{4}
\end{align*}

In the second line, we used the fact that $\lambda_{\max}(xx^T) = \|x\|^2$. To see this, note that $x x^T$ is rank one and $(x x^T)x = \|x\|^2 x$, therefore $x$ is an eigenvector with eigenvalue $\|x\|^2$ and this is the largest eigenvalue.  In the last line we used that $\|\phi''(s)\| \leq \frac{1}{4}$ for all $s$ and that $\|x\| \leq \boundx$.

\section{Experiment details}\label{app:experiments}

\begin{table}[htbp]
\centering
\caption{Dataset information}
\begin{tabularx}{\textwidth}{|p{3cm}|l|p{1.8cm}|X|p{2.5cm}|}
\hline
\textbf{Dataset} & \textbf{Size} & \textbf{\# Attributes} & \textbf{Attributes} & \textbf{Target} \\ \hline
Adult        & 48,842          & 15                           & ['age', 'workclass', 'fnlwgt', 'education', 'marital-status',
       'occupation', 'relationship', 'race', 'sex', 'capital-gain',
       'capital-loss', 'hours-per-week', 'native-country', 'income$>$50K',
       'education-num'] & 'income$>$50K' (logistic), 'education-num' (linear)         \\ \hline
ACSIncome        & 195,665          & 9                           & ['AGEP', 'COW', 'SCHL', 'MAR', 'RELP', 'WKHP', 'SEX', 'RAC1P', 'PINCP'] & 'PINCP'         \\ \hline
Fire        & 305,119          & 15                            & ['ALS Unit', 'Battalion', 'Call Final Disposition', 'Call Type',
       'Call Type Group', 'City', 'Final Priority', 'Fire Prevention District',
       'Neighborhooods - Analysis Boundaries', 'Original Priority',
       'Station Area', 'Supervisor District', 'Unit Type',
       'Zipcode of Incident', 'Priority'] & 'Priority'         \\ \hline
Taxi        & 1,048,575          & 11                            & ['VendorID', 'passengercount', 'tripdistance', 'RatecodeID',
       'PULocationID', 'DOLocationID', 'paymenttype', 'fareamount',
       'tipamount', 'tollsamount', 'totalamount'] & 'totalamount'         \\ \hline
ACSEmployment        & 378,817          & 17                            & ['AGEP', 'SCHL', 'MAR', 'RELP', 'DIS', 'ESP', 'CIT', 'MIG', 'MIL', 'ANC',
       'NATIVITY', 'DEAR', 'DEYE', 'DREM', 'SEX', 'RAC1P', 'ESR'] & 'ESR'         \\ \hline
ACSPublicCoverage        & 138,550          & 19                            & ['AGEP', 'SCHL', 'MAR', 'SEX', 'DIS', 'ESP', 'CIT', 'MIG', 'MIL', 'ANC',
       'NATIVITY', 'DEAR', 'DEYE', 'DREM', 'PINCP', 'ESR', 'FER', 'RAC1P',
       'PUBCOV'] & 'PUBCOV'         \\ \hline
\end{tabularx}
\label{tab:datasets}
\end{table}

Information about the datasets used in the Experiments section is reported in Table~\ref{tab:datasets}.

All experiments are run on an internal cluster using Xeon Gold 6240 CPU @ 2.60GHz, 192GB RAM and 240GB local SSD disk. Experiment runtime varies depending on multiple factors. Our method is based on private marginals released by \AIM{} and is therefore tied to \AIM{}'s runtime, which is significantly longer than the runtime of our DP baselines. For \AIM{}, runtime increases with the domain size and $\epsilon$. In the following runtime analysis, we refer to a full experiment as one trial for each $\epsilon$ in $[0.05, 0.1, 0.5, 1.0, 2.0]$, at 200MB model size and 1,000 maximum iterations for \AIM{}, and running on the internal cluster specified above. For linear regression, the longest \AIM{} experiment is for dataset Fire, taking approximately 29 hours. The shortest linear regression experiment with the same setting is ACSincome, at approximately 50 minutes runtime. By comparison, \AdaSSP{} completes the corresponding experiments in approximately 5 seconds (Fire) and 1 second (ACSincome). For logistic regression, the longest \AIM{} experiment is for ACSPublicCoverage, with a runtime of approximately 43 hours, and the shortest is ACSincome at approximately 10 hours. The corresponding runtime for DP baseline \ObjPert{} is approximately 2 seconds in both cases. \AIM{} computation time grows with $\epsilon$. For example, in the case of linear regression on Adult, running \AIM{} takes approximately 8 minutes with $\epsilon = 0.05$, and approximately 18 hours for $\epsilon = 2.0$.

\end{document}